\documentclass{article} 
\usepackage{iclr2027_conference,times}

\usepackage{amsmath,amsfonts,bm}

\def\eqref#1{equation~\ref{#1}}

\def\1{\bm{1}}

\DeclareMathAlphabet{\mathsfit}{\encodingdefault}{\sfdefault}{m}{sl}
\SetMathAlphabet{\mathsfit}{bold}{\encodingdefault}{\sfdefault}{bx}{n}

\newcommand{\E}{\mathbb{E}}

\usepackage{hyperref}
\usepackage{url}
\usepackage{fontawesome5}
\usepackage{enumitem}
\usepackage{graphicx} 
\usepackage{wrapfig}
\usepackage{fancyhdr}
\usepackage{booktabs}
\usepackage{marvosym}
\usepackage{multirow}
\usepackage{tcolorbox}
\usepackage{xcolor}
\definecolor{citeblue}{RGB}{0,80,160}
\definecolor{linkred}{RGB}{160,30,30}
\hypersetup{
  colorlinks=true,
  citecolor=citeblue,
  linkcolor=linkred,
  urlcolor=citeblue
}
\definecolor{takeawayframe}{HTML}{6F8F76}
\definecolor{takeawayback}{HTML}{F1F4EE}
\newtcolorbox{takeaway}{colframe=takeawayframe, colback=takeawayback,
  boxrule=0.5pt, arc=1.5pt, left=5pt, right=5pt, top=3pt, bottom=3pt}
\usepackage{amsthm}
\newtheorem{proposition}{Proposition}

\title{Understanding On-Policy Distillation: A Mechanistic Interpretability Perspective via Sparse Crosscoders}

{\footnotetext{Corresponding author.}}

\author{Zichao Yu\textsuperscript{1} \quad Qianshuo Ye\textsuperscript{2} \quad
Xu Wang\textsuperscript{1,3} \quad Difan Zou\textsuperscript{1,3,\Letter} \\
\textsuperscript{1}School of Computing and Data Science, The University of Hong Kong \\
\textsuperscript{2}Department of Computer Science and Technology, University of Cambridge \\
\textsuperscript{3}Shenzhen Loop Area Institute \\
\texttt{yuzichao1015@gmail.com}, \texttt{qy250@cam.ac.uk}, \\
\texttt{sunny615@connect.hku.hk}, \texttt{dzou@cs.hku.hk} \\
Project page: \faGlobe\ \url{https://yzc-666.github.io/understanding-opd-crosscoders/}
}
\iclrfinalcopy 
\begin{document}

\maketitle
\lhead{Preprint.}

\begin{abstract}
On-policy distillation (OPD) has emerged as a widely adopted post-training
technique for LLM reasoning. It is commonly believed to transfer knowledge from
a stronger teacher to the student, yet what OPD actually distills into the
student's internal representations remains unclear. In this study, we examine
this question through the lens of sparse crosscoders, which learn one feature
dictionary shared by the student before and after OPD and the teacher. Standard
crosscoder analyses, however, identify features specific to one model but
cannot tell how a model's use of its features changes, since the crosscoder
encodes all models into a single set of feature activations. We therefore
propose the swap readout, which reads the feature activations of each student
checkpoint on its own and thus measures how training changes the student's use
of each feature, even for checkpoints the crosscoder has never seen.
Specifically, across three OPD settings, we observe that OPD neither creates
features nor passes on the teacher's own, and leaves the firing rates of over
$98\%$ of the student's frequently used features within
$20\%$. These results suggest that OPD primarily reweights the features the
student already shares with the teacher rather than acquiring new ones. We
further examine the SFT warm-up on the teacher's rollouts that commonly precedes
OPD and makes it more effective. Rather than adding features, as one might
expect, the warm-up reweights the shared ones in two ways. First, it already
raises and lowers many of the features that OPD later raises and lowers, doing
part of OPD's work in advance. Second, it changes features in ways that OPD
alone would not, most notably those for the conversation format, the style of
reasoning, and mathematical notation, and these changes persist through OPD. Imposing this reweighting on the
features of a directly distilled student, without changing its weights, brings
its accuracy close to that of the warmed-up student, whereas the same change on
shuffled features does not. Together, these findings suggest
that OPD behaves more like a \emph{reweighting} of existing features than an
\emph{acquisition} of new ones: the student learns from the teacher how to use
the features they already share.
\end{abstract}

\section{Introduction}
\label{sec:intro}

\begin{figure}[t]
\centering
\includegraphics[width=\linewidth]{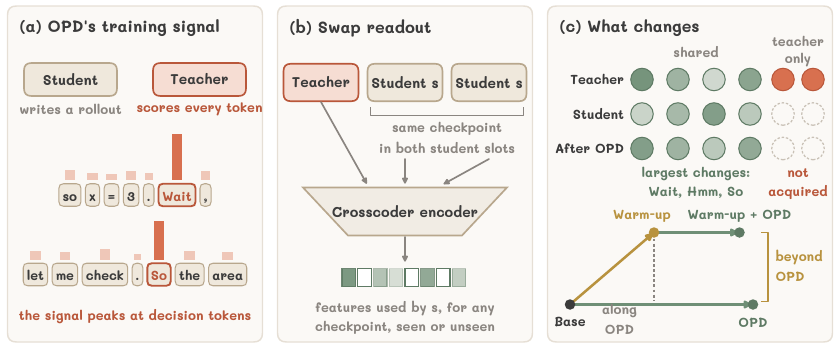}
\vskip -.1in
\caption{\textbf{Overview.} (a)~In OPD, the student writes a rollout and the
teacher scores every token; the signal is strongest at decision tokens. (b)~The swap
readout places a student checkpoint in both student slots of the crosscoder,
with the teacher's activation fixed, and reads the features the checkpoint
uses. (c)~OPD reweights the features the student shares with the teacher, most
at decision tokens, and acquires none of the teacher's own; the warm-up moves
the student partly along OPD's direction and partly beyond it.}
\label{fig:overview}
\vskip -16pt
\end{figure}

On-policy distillation
(OPD)~\citep{agarwal2024policy,gu2024minillm,lu2025onpolicydistillation} has
become a standard post-training paradigm for reasoning models, adopted by
Qwen3~\citep{yang2025qwen3}, MiMo-V2-Flash~\citep{xiao2026mimo},
GLM-5~\citep{zeng2026glm}, and Kimi~K3~\citep{team2026kimi}. It is widely
believed to let the student acquire new capabilities from a stronger
teacher~\citep{li2026rethinking}, yet the evidence is mixed: OPD can fail to
improve on, or even degrade, the student it starts
from~\citep{li2026rethinking,zhu2026many}, and it mainly improves sampling
efficiency rather than expanding the student's capability
boundary~\citep{ge2026towards}. This raises a basic question: \emph{what does
the student actually learn from the teacher in OPD?} Does it acquire something
new, or does it learn to use what it already has differently?

Existing analyses of OPD, including those above, study it through the student's
outputs, namely its token probabilities, accuracy, and training
signal~\citep{li2026rethinking,ding2026does,ge2026towards}, and do not reveal
what changes inside the student. We examine the question from the student's
internal representations, through the lens of sparse
crosscoders~\citep{lindsey2024sparse}, which learn one dictionary of features
for several models at once, so that the student before OPD, the student after
OPD, and the teacher can be compared feature by
feature~\citep{minder2026overcoming,shi2026does}. Standard crosscoder analyses
compare models through their decoders and identify the features specific to
one model. They cannot, however, reliably show which features each model
activates on a given input, since the crosscoder encodes all models jointly
into a single set of feature activations; \citet{minder2026overcoming} likewise
note that crosscoders provide no mechanism to track how a model's feature
activations change. We therefore propose the \emph{swap readout}
(Figure~\ref{fig:overview}b). It places the activation of a student checkpoint
in both student inputs of the
crosscoder while holding the teacher's input fixed, and thereby reads the
feature activations of that checkpoint on its own, even for checkpoints the
crosscoder has never seen.

We apply this readout to three OPD settings whose teachers differ from the
student by RL, by scale, or by both. In all three, OPD neither creates features
of the student's own nor passes on the teacher's
(Figure~\ref{fig:overview}c). It only changes how often the student uses the
features it already shares with the teacher, and only slightly: over $98\%$ of
the frequently used features change their firing rate by less than $20\%$.
Where OPD changes the student most, the features that change most fire on words
such as \emph{Wait}, \emph{Hmm}, and \emph{So}, at which a reasoning trace
decides its next move and at which the teacher disagrees with the student most
(Figure~\ref{fig:overview}a).
These findings give a representation-level account of behavioral findings on
OPD: if OPD can only reweight the features the student already shares with the
teacher, it should improve sampling efficiency without expanding the capability
boundary~\citep{ge2026towards}, and it works only when the student and the
teacher share compatible thinking patterns~\citep{li2026rethinking}.

If OPD only changes how the student uses what it already has, then what the
student starts from should matter. This brings us to a step that commonly
precedes OPD: an SFT warm-up on the teacher's rollouts, which makes OPD more
effective~\citep{yang2025qwen3,lu2025onpolicydistillation,li2026rethinking,liu2026simple}.
A natural explanation for this benefit is that the warm-up supplies what OPD
cannot, namely new features that bring the student closer to the teacher;
indeed, SFT has been found to introduce new features~\citep{shi2026does}. Following
the setup of Simple-OPD~\citep{liu2026simple}, we reproduce this benefit but
find the opposite: the warm-up creates no features and does not make the
teacher's features shared. Instead, it also reweights the shared features,
partly along OPD's direction, making part of OPD's change in advance, and
partly in directions that OPD does not take and that persist through OPD. This
reweighting offers a feature-level view of the teacher-compatible thinking
pattern to which Simple-OPD attributes the warm-up's benefit. Imposing this
reweighting on the features of a directly distilled student,
without changing its weights, recovers most of the warm-up's benefit, and
removing it from the warmed-up student removes most of it, whereas the same
change on shuffled features does neither. Together, these findings suggest
that OPD behaves more like a \emph{reweighting} of existing features than an
\emph{acquisition} of new ones.
Our contributions and findings are summarized as follows:
\begin{itemize}[leftmargin=*]
\item We propose the swap readout, which measures how training changes a
student's use of each crosscoder feature, even for checkpoints the crosscoder
has never seen; the change it identifies can be imposed on a student's features
and alters the student's accuracy. Unlike existing crosscoder analyses, which
identify features specific to one
model~\citep{lindsey2024sparse,minder2026overcoming,shi2026does}, it tracks how
a model's use of shared features changes.
\item We show that OPD distills no new features into the student; it only
reweights the features the student already shares with the teacher. Whereas
prior analyses of OPD examine the student's
outputs~\citep{li2026rethinking,ding2026does,ge2026towards}, this finding
concerns the student's internal representations.
\item We show that an SFT warm-up on the teacher's rollouts adds no features
either, unlike SFT on solutions written by a separate, stronger
model~\citep{shi2026does}; it reweights the shared features partly along OPD's direction and
partly beyond it, and imposing this reweighting on the features of a directly
distilled student, without changing its weights, brings its accuracy close to
that of the warmed-up student.
\end{itemize}

\section{Measuring Feature Usage with Sparse Crosscoders}
\label{sec:prelim}

\subsection{Sparse Crosscoders and Model Attribution}
\label{sec:prelim-background}

\textbf{Crosscoders.}
Sparse crosscoders~\citep{lindsey2024sparse} extend sparse autoencoders to model diffing: they learn one dictionary of features for several models at once, so that the models can be compared feature by feature. Given the activations $\mathbf h_m$ of models $m=1,\dots,M$ on the same token, a crosscoder encodes them jointly into a single sparse code $\mathbf z$ and reconstructs each model's activation from this code with that model's own decoder,
\begin{equation}
\label{eq:crosscoder}
\mathbf z=\sigma\Big(\sum_{m}\mathbf W_m\mathbf h_m+\mathbf b\Big),\qquad
\hat{\mathbf h}_m=\mathbf D_m\mathbf z,
\end{equation}
so that each feature has a single activation on a token but a separate decoder direction in every model, the corresponding column of $\mathbf D_m$. The crosscoder is trained to minimize the reconstruction error $\sum_m\|\mathbf h_m-\hat{\mathbf h}_m\|_2^2$, so the decoder direction of a feature in a model describes how the feature contributes to that model's activation. We use BatchTopK crosscoders~\citep{bussmann2024batchtopk}, which enforce sparsity through $\sigma$ rather than an $\ell_1$ penalty, whose artifacts can make shared features appear model-specific~\citep{minder2026overcoming}. During training, $\sigma$ keeps, on average over a batch, the $k$ features per token with the largest activations weighted by their decoder norms; at inference, a threshold estimated during training replaces this batch-level selection. We write $\operatorname{enc}$ for the encoder with this threshold and say that a feature fires on a token when its entry of $\mathbf z$ is positive.

\textbf{Model attribution.}
Following \citet{shi2026does}, we measure how much a feature belongs to each model by its share of the decoder norm. The model attribution score (MAS) of feature $j$ for model $m$ is
\begin{equation}
\label{eq:mas}
\mathrm{MAS}(m,j)=\frac{\|\mathbf d_{mj}\|_1/\sqrt{n_m}}{\sum_{m'}\|\mathbf d_{m'j}\|_1/\sqrt{n_{m'}}},
\end{equation}
where $\mathbf d_{mj}$ is the decoder direction of feature $j$ in model $m$, the $j$-th column of $\mathbf D_m$, and dividing by the square root of the hidden size $n_m$ corrects for models of different width.
A feature used equally by $M$ models has an MAS of $1/M$ for each, and we say that model $m$ dominates a feature when its MAS exceeds $0.5$.
For two models and without the width correction, the MAS of the second model is the normalized relative norm (NRN) of \citet{shi2026does}: it is near $1$ for a feature specific to the second model, near $0$ for one specific to the first, and $1/2$ for a shared one.

\subsection{The Swap Readout: Tracking Feature Usage across Checkpoints}
\label{sec:prelim:swap}
\textbf{Why existing readouts fall short.} We ask which features the student uses on each input and how training changes this use. Each crosscoder is trained jointly on the student before OPD ($B$), the student after OPD ($O$), and the teacher ($T$). Its \emph{joint encoding} gives one code for all three models and cannot tell which of them makes a feature fire. \emph{Decoder attribution} such as MAS barely differs between the two students, whose feature directions nearly coincide, whereas training changes \emph{when} the student uses a feature. Reading a student through \emph{its own slot}, with the other student slot empty, is also unreliable. The two students' activations differ by only $7$--$12\%$ of their norm (Table~\ref{tab:app-swap}), so training determines the sum $\mathbf S=\mathbf W_B+\mathbf W_O$ of their encoders but leaves their difference near its random initialization; in the own-slot readout, this undetermined part is as large as the determined one (Appendix~\ref{app:swap}). We therefore propose the swap readout.

\begin{takeaway}
\textbf{Swap readout.} To read a checkpoint $s$ of the student on input $x$, we place its activation in \emph{both} student slots and keep the teacher's activation in the teacher slot:
\begin{equation}
\label{eq:swap}
\mathbf z_s(x)=\operatorname{enc}\!\big(\bar{\mathbf h}_s(x),\,\bar{\mathbf h}_s(x),\,\mathbf h_T(x)\big),
\end{equation}
where $\bar{\mathbf h}_s$ is the student's activation rescaled to the mean norm of the student before training, since the inference threshold makes firing sensitive to the scale of the residual stream.
\end{takeaway}

\textbf{Why the swap readout reads the student.} With the same activation in both student slots, the undetermined difference of the two encoders cancels. For checkpoint $s$, the pre-activation, the argument of $\sigma$ in Eq.~\eqref{eq:crosscoder}, and its change when $s$ is replaced by another checkpoint $s'$ on the same input are
\begin{equation}
\label{eq:swap-preact}
\mathbf a_s(x)=\mathbf S\,\bar{\mathbf h}_s(x)+\mathbf c(x),\qquad
\mathbf a_{s'}(x)-\mathbf a_s(x)=\mathbf S\,\big(\bar{\mathbf h}_{s'}(x)-\bar{\mathbf h}_s(x)\big),
\end{equation}
with $\mathbf c(x)=\mathbf W_T\,\mathbf h_T(x)+\mathbf b$. The crosscoder thus acts as a sparse autoencoder of the student, with the learned encoder $\mathbf S$, the student's decoder, and a bias $\mathbf c(x)$ that the teacher sets for each token. The features that fire are those from which the student's decoder reconstructs the checkpoint's activation, which is what it means for the checkpoint to use them. Since every checkpoint is read with the same autoencoder and the same teacher input, a feature starts or stops firing only because the student's own activation moves along the feature's encoder direction, a row of $\mathbf S$, while the teacher's term cancels. Appendix~\ref{app:swap} shows that no other weighting of the two student slots removes the undetermined part while agreeing with the joint encoding when the two students coincide.

Reading the student before and after training on the same input gives, for every feature, one of four outcomes (Figure~\ref{fig:swap-outcomes}): the feature fires in \emph{neither} checkpoint, fires in \emph{both} with a different strength, \emph{starts} to fire, or \emph{stops} firing. Counting these outcomes over many inputs shows how often the student uses each feature and how training changes this use.

\begin{wrapfigure}{r}{0.44\linewidth}
\vspace{-12pt}
\centering
\includegraphics[width=\linewidth]{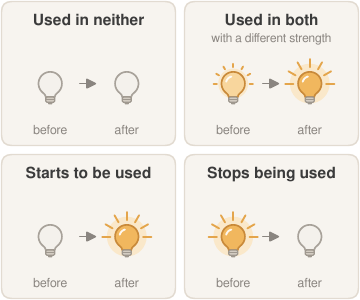}
\caption{\textbf{What the swap readout records for a feature on one input.} A lit bulb marks a feature that fires, and a brighter bulb marks a stronger activation.}
\label{fig:swap-outcomes}
\vspace{-12pt}
\end{wrapfigure}

\textbf{Why the swap readout is reliable.} (i)~\emph{Matched to training.} Its input, in which the two student slots agree, is close to the inputs on which the crosscoder is trained, so the encoder and the inference threshold are used under the conditions in which they were learned. (ii)~\emph{Faithful.} It reconstructs the student's activations as well as the joint encoding does, and it reconstructs checkpoints that the crosscoder never saw, such as the student after the SFT warm-up of Section~\ref{sec:further}, as well as those it was trained on (Appendix~\ref{app:crosscoders}). It thus reads the activation it is given, not the model a slot was trained on.

\textbf{Feature statistics.} We compute feature statistics with this readout on the same $3$M held-out tokens for every model. A feature's \emph{firing rate} is the share of tokens on which it fires, and the change in its use is the \emph{log ratio} of its firing rates after and before training. A feature is \emph{gained} or \emph{lost} when it fires at least $10$ times in one checkpoint and never in the other. For each token, the change is the share of features active before or after training that are not active in both. For a teacher of the student's width, we read its firing rates by placing its activation in all three slots.

\section{What Changes Inside the Student during OPD?}
\label{sec:analysis}

\begin{table}[t]
\centering
\caption{Detailed information of three OPD settings.}
\label{tab:settings}
\begin{tabular*}{\linewidth}{@{\extracolsep{\fill}}lllc@{}}
\toprule
Setting & Base student model & Teacher model & Dimension (student / teacher) \\
\midrule
JustRL  & R1-Distill-Qwen-1.5B & JustRL-DeepSeek-1.5B & 1536 / 1536 \\
Skywork & R1-Distill-Qwen-1.5B & Skywork-OR1-Math-7B  & 1536 / 3584 \\
R1-7B   & R1-Distill-Qwen-1.5B & R1-Distill-Qwen-7B   & 1536 / 3584 \\
\bottomrule
\end{tabular*}
\end{table}

In this section, we ask what OPD distills into the student. We study the three
OPD settings in Table~\ref{tab:settings}, in which a
DeepSeek-R1-Distill-Qwen-1.5B student~\citep{guo2025deepseek} is distilled from
one of three teachers, and we name each setting after its teacher: JustRL for
JustRL-DeepSeek-1.5B~\citep{he2025justrl}, Skywork for
Skywork-OR1-Math-7B~\citep{he2025skywork}, and R1-7B for
DeepSeek-R1-Distill-Qwen-7B~\citep{guo2025deepseek}. All students are trained on
DAPO-Math-17k~\citep{yu2026dapo} with the vanilla OPD recipe of
\citet{li2026rethinking} under the same hyperparameters, and all models are
evaluated on AIME 2024, AIME 2025, and AMC 2023. For each setting, we train one
crosscoder jointly on the base student, the OPD student, and the teacher, using
their residual-stream activations at a middle layer on a mixture of reasoning
and general text, and compare the students through it with the swap readout of
Section~\ref{sec:prelim:swap}. Appendix~\ref{app:experimental-details} gives the
full training, evaluation, and crosscoder details. We first ask whether OPD
gives the student new features (Section~\ref{sec:analysis-reweight}) and then
which of the student's features it reweights (Section~\ref{sec:analysis-local}).

\subsection{OPD Reweights Shared Features Rather Than Acquiring New Ones}
\label{sec:analysis-reweight}

\begin{figure}[t]
\centering
\includegraphics[width=\linewidth]{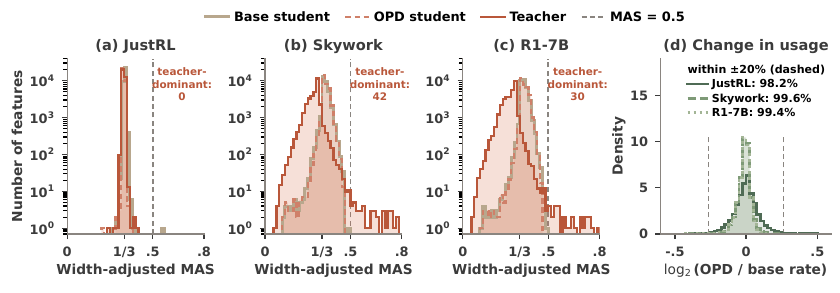}
\vskip -.1in
\caption{\textbf{OPD reweights the student's features and acquires none of the
teacher's.} (a--c)~Width-adjusted MAS of all features; dashed: MAS of $0.5$,
above which one model dominates a feature. (d)~Change in feature firing rates
after OPD; dashed: $\pm 20\%$.}
\label{fig:feature-reweighting}
\vskip -15pt
\end{figure}

Each crosscoder is trained on the
OPD student alongside the base student and the teacher, yet in none of the three
settings does the OPD student dominate a single feature, and its MAS
distribution coincides with that of the base student
(Figure~\ref{fig:feature-reweighting}a--c). The student also uses the same
features to the same extent before and after OPD. In the swap readout on $3$M
held-out tokens, no feature that fires at least $10$ times before OPD is absent
after it, or vice versa, and $98.2\%$, $99.6\%$, and $99.4\%$ of the frequently
used features change their firing rate by less than $20\%$
(Figure~\ref{fig:feature-reweighting}d).

We next turn to the teacher's own features. The JustRL teacher, an RL fine-tune
of the base student, dominates no feature. The Skywork and R1-7B teachers
dominate $42$ and $30$ features (Figure~\ref{fig:feature-reweighting}b,c;
Appendix~\ref{app:feature-examples}). If OPD passed these features on to the
student, the OPD student would take a larger share of their decoder norm, and
they would become shared. It does not: on these features, the student's MAS is
$0.19$ both before and after OPD in both settings, and no feature's MAS changes
by more than $0.015$. The teacher's own features stay with the teacher.

Since OPD neither creates features
nor acquires the teacher's, what it changes is how the student uses the shared
features it already has: which of them fire on a given token, and how strongly.
This reweighting is small: about one in ten of the features active on a token
changes, while each feature's overall usage stays the same
(Figure~\ref{fig:feature-reweighting}d). We characterize this reweighting in
Section~\ref{sec:analysis-local} and examine what an SFT warm-up before OPD
changes in Section~\ref{sec:further}.

\begin{takeaway}
\textbf{Takeaway.} OPD distills no new features into the student and passes on
none of the teacher's own; it only reweights the features the student already
shares with the teacher.
\end{takeaway}

\subsection{Decision Tokens Mark the Largest Feature Changes}
\label{sec:analysis-local}

\begin{figure}[t]
\centering
\includegraphics[width=\linewidth]{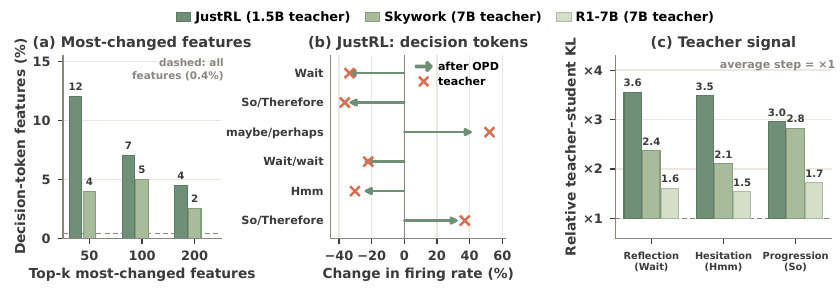}
\vskip -.1in
\caption{\textbf{The features OPD changes most fire on decision tokens.}
(a)~Share of decision-token features among the $k$ most-changed features;
dashed: among all features.
(b)~Under JustRL, change in firing rate of the decision-token features among
the most-changed ones, after OPD (arrows) and in the teacher (crosses).
(c)~Teacher--student KL at the steps that emit each category of decision token
on the student's rollouts, relative to the average step.}
\label{fig:local-reweighting}
\vskip -12pt
\end{figure}

Section~\ref{sec:analysis-reweight} showed that OPD changes each feature's
usage only slightly: $98.2\%$, $99.6\%$, and $99.4\%$ of the frequently used
features change their firing rate by less than $20\%$ under JustRL, Skywork,
and R1-7B. We now ask which features make up the small remainder, reading each
feature's meaning off the tokens on which it fires most strongly
(Appendix~\ref{app:feature-examples}). Under JustRL, the features whose firing rate changes most are dominated by features
for the discourse words that steer a reasoning trace: a feature that fires on
\emph{Wait} fires $35\%$ less often after OPD, one that fires on \emph{So} and
\emph{Therefore} $35\%$ less, and one that fires on \emph{maybe} and
\emph{perhaps} $42\%$ more. Such features make up $0.4\%$ of all features but
$12\%$ of the $50$ whose firing rate changes most
(Figure~\ref{fig:local-reweighting}a). These words mark the points at which a
trace chooses its next move, so we call them decision tokens and group them by
that move into three categories. \emph{Reflection} tokens revisit an earlier
step (\emph{Wait}, \emph{actually}); \emph{hesitation} tokens express doubt or
turn to another approach (\emph{Hmm}, \emph{maybe}, \emph{Alternatively}); and
\emph{progression} tokens move on to the next step or draw a conclusion
(\emph{So}, \emph{Therefore}, \emph{Let}; Appendix~\ref{app:decision} gives the
full lists).

The concentration is strongest where OPD changes the student most.
Decision-token features make up $12\%$ of the $50$ most-changed ones under
JustRL and $4\%$ under Skywork, against a $0.4\%$ baseline; under R1-7B, none
of the $50$ most-changed features fires on a decision token
(Figure~\ref{fig:local-reweighting}a). OPD also moves these features to the
teacher's usage. Under JustRL, the \emph{Wait} feature fires $35\%$ less often
after OPD and $34\%$ less often in the teacher, the \emph{So}/\emph{Therefore}
feature $35\%$ and $36\%$ less, and the \emph{maybe}/\emph{perhaps} feature
$42\%$ and $52\%$ more (Figure~\ref{fig:local-reweighting}b).

Because these features fire at decision tokens, the reweighting is local along
the trace as well. At decision tokens, the student under JustRL and Skywork
replaces $10$--$44\%$ more of its active features than at an average token, and
OPD's change to the student's next-token distribution concentrates at the
steps that emit them, which are $4$--$7$ times over-represented among the
$5\%$ most-changed steps. These are also the steps at which the teacher
disagrees with the student most: the KL divergence from the teacher's
next-token distribution to the student's is $3.0$--$3.6$, $2.1$--$2.8$, and
$1.5$--$1.7$ times its average there under JustRL, Skywork, and R1-7B
(Figure~\ref{fig:local-reweighting}c). OPD thus receives its strongest signal
where the trace decides how to proceed and reweights the features that mark
that decision. Under R1-7B, OPD barely moves the student at all: its
next-token distribution departs from the base student's by a KL of only
$0.008$ on average, a tenth of the $0.085$ under JustRL, so its features
remain essentially unchanged.
\begin{takeaway}
\textbf{Takeaway.} Where OPD changes the student most, the features it changes
most fire on decision tokens such as \emph{Wait}, \emph{Hmm}, and \emph{So}, at
which a trace reflects, hesitates, or moves on. These are the tokens at which
the teacher disagrees with the student most, and OPD moves the firing rates of
these features toward those of the teacher.
\end{takeaway}

\section{Why an SFT Warm-Up Helps OPD}
\label{sec:further}

Section~\ref{sec:analysis} showed that OPD distills no new features into the
student: it only reweights the features the student already shares with the
teacher, most of all at decision tokens. We now apply the same analysis to a
standard step before OPD, an SFT warm-up on the teacher's rollouts, which is
known to make OPD work better~\citep{li2026rethinking,liu2026simple}, and ask
what this warm-up changes in the student. Following the setup of
Simple-OPD~\citep{liu2026simple}, we distill Qwen3-1.7B-Base from
Qwen3-4B-Base-GRPO~\citep{li2026rethinking} and warm up the student on the
teacher's own rollouts before OPD (Appendix~\ref{app:experimental-details}).

\begin{table}[t]
\centering
\small
\caption{\textbf{The SFT warm-up makes OPD work better.} avg@8 (\%), the
accuracy averaged over 8 samples per problem, of the student distilled from
Qwen3-4B-Base-GRPO with and without the warm-up. Recovered: share of the
teacher's advantage over the base student that the student recovers.}
\label{tab:sft-warmup}
\begin{tabular*}{\linewidth}{@{\extracolsep{\fill}}lccccc@{}}
\toprule
 & AIME24 & AIME25 & AMC23 & Avg. & Recovered \\
\midrule
Base student  & 4.6  & 2.1  & 22.5 & 9.7  & --    \\
OPD           & 7.9  & 4.6  & 40.3 & 17.6 & 30\% \\
Warm-up + OPD & 12.1 & 9.6  & 44.4 & 22.0 & 46\% \\
Teacher       & 22.9 & 20.0 & 65.9 & 36.3 & --    \\
\bottomrule
\end{tabular*}
\end{table}

\textbf{The warm-up improves OPD.} With warm-up, OPD recovers more of the teacher's advantage over the base student and improves on every benchmark (Table~\ref{tab:sft-warmup}), confirming prior findings. We next test whether the warm-up adds features, as SFT can~\citep{shi2026does} (Section~\ref{sec:sft-features}), and, finding none, examine what changes instead (Section~\ref{sec:sft-reweight}).

\subsection{The Warm-up Gives the Student No New Features Either}
\label{sec:sft-features}

\begin{figure}[t]
\centering
\includegraphics[width=\linewidth]{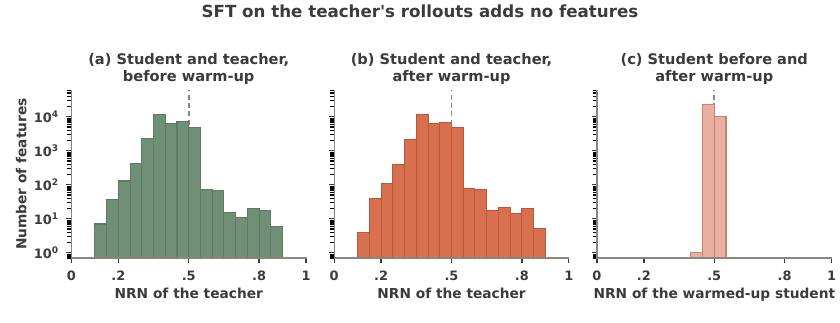}
\caption{\textbf{SFT on teacher rollouts induces no model-specific features.}
NRN~\citep{shi2026does} in two-model crosscoders: a model's share of a
feature's decoder norm, near $1$ if the feature is specific to that model and
$1/2$ (vertical line) if shared. (a,b)~NRN of the teacher against the student
before and after the warm-up. (c)~NRN of the student after the warm-up against
the student before it.}
\label{fig:sft-nrn}
\end{figure}

\textbf{The warm-up creates no features of its own.} In a two-model crosscoder of
the student before and after the warm-up, the test in which SFT produces many
features specific to the fine-tuned model~\citep{shi2026does}, the warm-up
produces none: every feature keeps an NRN close to $1/2$
(Figure~\ref{fig:sft-nrn}c).
Read with the swap readout through the crosscoder trained before the warm-up,
the student after the warm-up also gains no feature, and the frozen crosscoder
reconstructs it as well as the base student.

\textbf{It also does not acquire the teacher's feature.} If the warm-up gave the student
features of the teacher, these features would become shared in the crosscoder
of the student and the teacher: their NRN, the teacher's share of the decoder
norm, would move from near $1$ toward $1/2$, and the tail of teacher-specific
features would thin. It does not. The teacher's NRN is distributed the same way
before and after the warm-up (Figure~\ref{fig:sft-nrn}a,b), with the same mean
($0.42$) and nearly the same number of teacher-specific features above $0.8$
($24$ and $25$). Likewise, in the three-model crosscoder, the features the
teacher dominates stay few ($19$ before the warm-up and $17$ after it) and carry
the same small share of the student's activation.

This differs from \citet{shi2026does}, whose SFT, like a conventional cold start
before RL, trains on solutions written by a separate, stronger model. The
warm-up instead stays within the task that OPD then trains on, math reasoning on
the same questions, and imitates the rollouts of the very teacher that OPD
distills from. Our conclusions concern this warm-up rather than SFT in general.

\begin{takeaway}
\textbf{Takeaway.} Training on the teacher's own rollouts does not give the
student the teacher's features: like OPD, the warm-up adds no features and only
reweights those the student already shares with the teacher.
\end{takeaway}

\subsection{The Warm-Up Anticipates and Complements OPD’s Reweighting}
\label{sec:sft-reweight}

\begin{figure}[t]
\centering
\includegraphics[width=\linewidth]{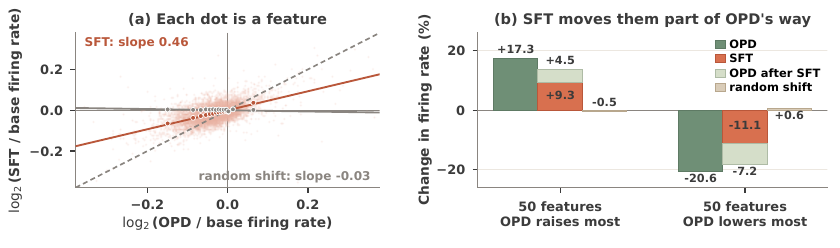}
\caption{\textbf{The SFT warm-up pre-reweights shared features in the direction
of OPD.} (a)~Change in each feature's firing rate after the warm-up against
that after OPD; circles: binned means; dashed: $y=x$, where the warm-up would
make all of OPD's change; solid: fitted slopes for the warm-up and for a random
shift. (b)~Mean change of the features OPD raises and lowers most; stacked: the
extra change OPD adds after the warm-up.}
\label{fig:sft-alignment}
\vskip -16pt
\end{figure}

If the warm-up can only reweight shared features, the question becomes how its
reweighting relates to OPD's. For each feature that fires at least $1{,}000$
times, we compare the change in its firing rate from the base student to the
student after the warm-up with the change to the student after OPD, both read
with the swap readout (Section~\ref{sec:prelim:swap}) through the crosscoder
trained before the warm-up (Figure~\ref{fig:sft-alignment}). We collect these
log changes in firing rate over features into two vectors, $\mathbf o$ for OPD
and $\mathbf s$ for the warm-up, and fit $\mathbf s \approx \beta\,\mathbf o$ by
least squares through the origin. The slope $\beta$ measures how far the
warm-up moves the features along OPD's reweighting. The residual
$\mathbf s - \beta\,\mathbf o$, which is orthogonal to $\mathbf o$, is the part
of the warm-up's change that lies outside OPD's direction.

\textbf{The warm-up moves the features OPD moves.} The two
changes point the same way: along OPD's direction, the warm-up moves the
features about half as far as OPD ($\beta = 0.46$, Figure~\ref{fig:sft-alignment}a), and $88\%$ and $92\%$ of
the $50$ features OPD raises and lowers most move the same way after the
warm-up. For the $50$ features OPD raises most, the warm-up already raises the
firing rate by $9.3\%$ against $17.3\%$ for OPD, and OPD after the warm-up adds
another $4.5$ points (Figure~\ref{fig:sft-alignment}b); for the $50$ features
OPD lowers most, the warm-up lowers it by $11.1\%$ against $20.6\%$, and OPD
adds another $7.2$ points. To check that this alignment reflects the direction of the warm-up's change, not just any change to the student's activations, we rotate the warm-up's change to a random direction; this \emph{random shift} produces no alignment ($\beta = -0.03$; Appendix~\ref{app:decomposition}).

\textbf{It also makes reweighting that OPD does not.} More than half of the
warm-up's change lies outside OPD's direction: the residual carries $56\%$ of
$\|\mathbf s\|^2$. This share excludes a single feature that fires on a garbled
character in web text and that the warm-up nearly silences; with it, the share
rises to $77\%$ (Appendix~\ref{app:decomposition}). This part survives the
subsequent OPD: computing the residual in the same way for the student after the
warm-up and OPD gives one that matches the residual after the warm-up alone
(Spearman $0.69$ across features, against $0.00$ for the random shift). It is
spread over many features, and half of it falls on features that fire mainly on
reasoning text. The features it moves most concern the conversation format (the
system prompt, the instruction on the answer format, and the turn boundaries),
the style of reasoning (features for \emph{Wait}, \emph{Alternatively}, and
\emph{But} fire less, and one for laying out a plan step by step fires more),
and mathematical notation (Appendix~\ref{app:feature-examples}). The student
after the warm-up and OPD thus combines
most of OPD's reweighting with the warm-up's own reweighting, and it recovers
$46\%$ of the teacher's advantage instead of $30\%$.

\begin{table}[t]
\centering
\small
\caption{\textbf{The reweighting found by the swap readout carries the warm-up's
benefit.} avg@8 (\%) of students with the warm-up's feature change added at the
crosscoder's layer or removed from it, with the difference $\Delta$ from the
unmodified student and its bootstrap interval, and their reweighting read with
the swap readout. All rows are sampled in the same way, which differs from
Table~\ref{tab:sft-warmup} (Appendix~\ref{app:feature-intervention}).}
\label{tab:feature-intervention}
\begin{tabular*}{\linewidth}{@{\extracolsep{\fill}}lccccccc@{}}
\toprule
& \multicolumn{5}{c}{avg@8 (\%)} & \multicolumn{2}{c}{Swap readout} \\
\cmidrule(lr){2-6}\cmidrule(l){7-8}
Student & AIME24 & AIME25 & AMC23 & Avg. & $\Delta$ & $\beta$ & $\rho$ \\
\midrule
OPD                       & 10.4 & 7.5 & 37.8 & 18.6          & --                                                   & 1.00 & --      \\
\quad $+$ feature change  & 11.2 & 7.1 & 45.6 & \textbf{21.3} & $\mathbf{+2.7}$\,{\scriptsize$[+0.5, +6.0]$} & 1.32 & $+0.93$ \\
\quad $+$ shuffled change & 7.5  & 4.6 & 35.6 & 15.9          & $-2.7$\,{\scriptsize$[-5.4, +0.3]$}          & 0.95 & $-0.04$ \\
\midrule
Warm-up + OPD             & 12.5 & 9.2 & 44.1 & 21.9          & --                                                   & 0.84 & $+0.69$ \\
\quad $-$ feature change  & 10.0 & 7.1 & 40.6 & \textbf{19.2} & $\mathbf{-2.7}$\,{\scriptsize$[-5.3, -0.4]$} & 0.20 & $-0.48$ \\
\quad $-$ shuffled change & 14.2 & 9.2 & 46.9 & 23.4          & $+1.5$\,{\scriptsize$[-0.8, +4.1]$}          & 0.85 & $+0.66$ \\
\bottomrule
\end{tabular*}
\end{table}

\textbf{The reweighting found by the swap readout carries the warm-up's
benefit.} To test whether this reweighting causally accounts for the warm-up's
benefit, we intervene on the students' features directly, leaving their weights
unchanged (Table~\ref{tab:feature-intervention}). At every token, we decode the
difference between the swap codes of the student after and before the warm-up
through the crosscoder's student decoder, and add this change to the residual
stream of the directly distilled student at the crosscoder's layer, or subtract
it from that of the student distilled after the warm-up
(Appendix~\ref{app:feature-intervention}). Adding the change raises the avg@8
of the directly distilled student from $18.6\%$ to $21.3\%$, close to the
$21.9\%$ of the warmed-up student, and removing it lowers the avg@8 of the
warmed-up student to $19.2\%$, close to direct OPD ($\Delta=+2.7$ and $-2.7$
points; both $95\%$ bootstrap intervals exclude zero). The swap readout
confirms that the intervention transfers the reweighting itself: the directly
distilled student moves further along OPD's direction and takes on the
warm-up's reweighting beyond it, whereas the warmed-up student loses both. A
control that shuffles the feature identities of the change leaves the
reweighting nearly intact and neither helps the directly distilled student nor
hurts the warmed-up one. The benefit therefore depends on which shared features
are reweighted, which is exactly what the swap readout identifies.

This answers the question we began with. Both the warm-up and OPD reweight
features the student already shares with the teacher. The warm-up helps OPD not
by giving the student features that OPD cannot, but through this reweighting,
which does part of OPD's reweighting in advance and adds reweighting that OPD
would not make on its own.

\begin{takeaway}
\textbf{Takeaway.} An SFT warm-up helps OPD not by giving the student new
features, but by reweighting the shared ones, partly as OPD would and partly in
ways OPD would not.
\end{takeaway}

\section{Conclusion}
\label{sec:conclusion}

This work studies what on-policy distillation (OPD) distills into the student
through the lens of sparse crosscoders. To this end, we propose the swap
readout, which measures how training changes a student's use of each crosscoder
feature, even for checkpoints the crosscoder has never seen. Across three OPD
settings, we find that OPD creates no features of the student's own and passes
on none of the teacher's; instead, it slightly reweights the features the
student already shares with the teacher, and where it changes the student most,
this reweighting concentrates on the tokens at which the trace decides its next
move. The SFT warm-up on the teacher's rollouts that commonly precedes OPD adds
no features either. It reweights the shared features partly along OPD's
direction and partly beyond it, and imposing this reweighting on the features
of a directly distilled student, without changing its weights, brings its
accuracy close to that of the warmed-up student.
These findings suggest that OPD behaves more like a \emph{reweighting} of
existing features than an \emph{acquisition} of new ones: what the student
learns from the teacher is how to use the features they already share.

\subsection*{AI use statement}

We used generative AI tools, including an LLM-based coding assistant, in
several parts of this work. For the manuscript, they drafted and revised text,
which we then edited. For the experiments, they wrote, debugged, and ran
analysis and plotting code, including the code for the swap-readout analyses,
the feature-level intervention, and the figures, and they suggested some
analyses and controls, which we evaluated before adopting. The research
questions, the selection of experiments, and the conclusions are our own. We
reviewed all AI-assisted text, checked the code and its outputs against the
reported results, and take full responsibility for the content of this paper,
including any errors.

\bibliography{iclr2027_conference}
\bibliographystyle{iclr2027_conference}

\appendix
\clearpage
\noindent\textbf{Roadmap.}
\begin{itemize}[leftmargin=*,itemsep=1pt,topsep=2pt]
\item Appendix~\ref{app:related} reviews related work on the mechanisms and failure modes of OPD and on model diffing with sparse crosscoders.
\item Appendix~\ref{app:swap} derives why the swap readout isolates the change of the student, which the other readouts of Section~\ref{sec:prelim:swap} cannot, and verifies its premises on our crosscoders.
\item Appendix~\ref{app:experimental-details} describes how we train the OPD students (\ref{app:opd}) and the SFT warm-up (\ref{app:warmup}), how we evaluate them (\ref{app:eval}), and how we train and validate the crosscoders (\ref{app:crosscoders}).
\item Appendix~\ref{app:analysis-details} describes how we compute the reported quantities: the feature statistics read with the swap readout (\ref{app:feature-stats}), the decision tokens and policy changes of Section~\ref{sec:analysis-local} (\ref{app:decision}), and the decomposition of the warm-up's reweighting (\ref{app:decomposition}) and the feature-level intervention (\ref{app:feature-intervention}) of Section~\ref{sec:sft-reweight}.
\item Appendix~\ref{app:feature-examples} visualizes the features behind our findings: the swap readout on held-out passages (\ref{app:swap-trace}), the features the teachers dominate (\ref{app:cards-teacher}), the decision-token features that OPD reweights (\ref{app:cards-decision}), and the features the warm-up moves along (\ref{app:cards-along}) and beyond (\ref{app:cards-warmup}) OPD's direction.
\end{itemize}

\section{Related Work}
\label{app:related}
\textbf{Mechanistic Understanding and Failure Modes of OPD.}
On-policy distillation (OPD) trains a student on its own trajectories while using a teacher to provide dense token-level supervision at visited states, thereby reducing the exposure bias associated with fixed demonstrations~\citep{agarwal2024policy,lu2025onpolicydistillation,bengio2015scheduled,ross2011reduction,gu2024minillm,yu2026mismatch}. Recent studies have moved beyond evaluating OPD's empirical effectiveness to examining its optimization mechanism, emphasizing teacher--student overlap and the conditions under which OPD succeeds or fails~\citep{li2026rethinking,zhu2026many}. However, a locally small teacher--student discrepancy does not necessarily lead to globally desirable rollouts: OPD can exhibit length inflation and repetition, enter KL-based agreement traps, or receive unreliable supervision on student-induced prefixes~\citep{luo2026demystifying,xin2026escaping,liu2026your}; related issues also arise when privileged teacher context alters the learning signal for long reasoning traces~\citep{kaur2026rethinking}. These findings suggest that teacher--student agreement can be locally uninformative. In particular, a repetitive student prefix may condition the teacher to favor the same continuation, producing near-zero local discrepancy despite a degenerate complete response. 


\textbf{Model Diffing with Sparse Crosscoders.}
Sparse autoencoders decompose neural activations into sparse combinations of learned features~\citep{cunningham2023sparse,bricken2023towards}. Sparse crosscoders extend this approach to compare shared and model-specific features across models~\citep{lindsey2024sparse}. \citet{minder2026overcoming} identify sparsity artifacts that make shared features appear model-specific and develop Latent Scaling and BatchTopK crosscoders to address them. \citet{jiralerspong2026cross} introduce dedicated features for comparisons across architectures. Applications to reasoning models reveal features associated with self-reflection and verification in distilled models~\citep{baek2025towards}, as well as features that modulate \emph{wait} tokens and subsequent reasoning patterns~\citep{troitskii2025internal}. \citet{shi2026does} compare SFT and RL, finding that SFT introduces specialized features while RL largely preserves base-model representations. We study how OPD and its teacher-trajectory SFT warm-up change the student's feature usage, using a swap readout to compare checkpoints within a fixed crosscoder.
\section{Theoretical Analysis of the Swap Readout}
\label{app:swap}

\paragraph{Setup.}
Let $\mathbf h_B(x)$, $\mathbf h_O(x)$, and $\mathbf h_T(x)$ be the activations of the student before training ($B$), the student after training ($O$), and the teacher ($T$) on input $x$ after the normalization of the crosscoder (Appendix~\ref{app:crosscoders}), let $\mathbf W_B$, $\mathbf W_O$, and $\mathbf W_T$ be the encoder matrices and $\mathbf b$ the encoder bias, and let $\mathbf d_{mj}$ be the decoder direction of feature $j$ in model $m$. The encoder computes the pre-activation
\begin{equation}
\label{eq:app-preact}
\mathbf a(x)=\mathbf W_B\,\mathbf h_B(x)+\mathbf W_O\,\mathbf h_O(x)+\mathbf W_T\,\mathbf h_T(x)+\mathbf b,
\end{equation}
and at inference feature $j$ fires when $\operatorname{ReLU}(a_j(x))\sum_m\|\mathbf d_{mj}\|_2>\theta$, that is, when $a_j(x)$ exceeds a feature-specific threshold $\theta_j$. We reparametrize the two student slots by $\mathbf S=\mathbf W_B+\mathbf W_O$ and $\mathbf N=\frac12(\mathbf W_B-\mathbf W_O)$, and write $\mathbf m=\frac12(\mathbf h_B+\mathbf h_O)$ and $\boldsymbol\delta=\mathbf h_B-\mathbf h_O$ for the mean and the difference of the two students, so that, exactly,
\begin{equation}
\label{eq:app-split}
\mathbf W_B\,\mathbf h_B+\mathbf W_O\,\mathbf h_O=\mathbf S\,\mathbf m+\mathbf N\,\boldsymbol\delta.
\end{equation}
With $\mathbf c(x)=\mathbf W_T\,\mathbf h_T(x)+\mathbf b$, the swap readout of a checkpoint $s$ places its rescaled activation $\bar{\mathbf h}_s$ in both student slots:
\begin{equation}
\label{eq:app-swap}
\mathbf a_s(x)=\mathbf W_B\,\bar{\mathbf h}_s(x)+\mathbf W_O\,\bar{\mathbf h}_s(x)+\mathbf c(x)=\mathbf S\,\bar{\mathbf h}_s(x)+\mathbf c(x).
\end{equation}

\paragraph{Training leaves $\mathbf N$ largely undetermined.}
By Eq.~\eqref{eq:app-split}, the pre-activation depends on $\mathbf S$ through $\mathbf S\mathbf m$ and on $\mathbf N$ only through $\mathbf N\boldsymbol\delta$. With $\mathbf g=\partial\ell/\partial\mathbf a$ the gradient of the loss with respect to the pre-activation, the chain rule gives $\nabla_{\mathbf S}\E[\ell]=\E[\mathbf g\,\mathbf m^{\top}]$ and $\nabla_{\mathbf N}\E[\ell]=\E[\mathbf g\,\boldsymbol\delta^{\top}]$: the gradient on $\mathbf N$ is the gradient on $\mathbf S$ with the mean of the two students replaced by their difference, which is an order of magnitude smaller (Table~\ref{tab:app-swap}). Because $\mathbf W_B$ and $\mathbf W_O$ are initialized independently, $\mathbf N$ starts as a random matrix with about half the norm of $\mathbf S$, and training, which barely acts on it, leaves much of this random part in place. Training therefore determines $\mathbf S$, but not $\mathbf N$ beyond its action on the small differences $\boldsymbol\delta$.

\begin{proposition}[Swap readout]
\label{prop:app-swap}
(i) The swap pre-activation $\mathbf a_s(x)$ is unchanged when $(\mathbf W_B,\mathbf W_O)$ is replaced by $(\mathbf W_B+\mathbf E,\mathbf W_O-\mathbf E)$ for any matrix $\mathbf E$, and among the readouts $\lambda\,\mathbf W_B\,\bar{\mathbf h}_s+\mu\,\mathbf W_O\,\bar{\mathbf h}_s+\mathbf c$, it is the only one with this invariance that equals the joint pre-activation of Eq.~\eqref{eq:app-preact} whenever $\mathbf h_B=\mathbf h_O=\bar{\mathbf h}_s$. (ii) For two checkpoints $s$ and $s'$ on the same input, $\mathbf a_{s'}(x)-\mathbf a_s(x)=\mathbf S\,\big(\bar{\mathbf h}_{s'}(x)-\bar{\mathbf h}_s(x)\big)$, and for the two students, without rescaling, the joint pre-activation is $\mathbf a(x)=\frac12\big(\mathbf a_B(x)+\mathbf a_O(x)\big)+\mathbf N\,\boldsymbol\delta(x)$.
\end{proposition}

\begin{proof}
(i) The replacement leaves $\mathbf S$ unchanged, and $\mathbf a_s$ depends on the student encoders only through $\mathbf S$. Since $\lambda\mathbf W_B+\mu\mathbf W_O=\frac{\lambda+\mu}{2}\,\mathbf S+(\lambda-\mu)\,\mathbf N$, invariance for every $\mathbf E$ requires $\lambda=\mu$, and agreement with the joint pre-activation $\mathbf S\,\bar{\mathbf h}_s+\mathbf c$ requires $\lambda+\mu=2$. (ii) The term $\mathbf c(x)$ is the same for both checkpoints and cancels. By Eq.~\eqref{eq:app-swap}, $\frac12(\mathbf a_B+\mathbf a_O)=\mathbf S\,\mathbf m+\mathbf c$, and by Eq.~\eqref{eq:app-split}, $\mathbf a=\mathbf S\,\mathbf m+\mathbf N\,\boldsymbol\delta+\mathbf c$.
\end{proof}

Part (i) separates the swap readout from the alternatives. Reading a student through its own slot, $(\lambda,\mu)=(1,0)$, applies the undetermined $\mathbf N$ to the whole activation and halves the determined contribution, and replacing only one student slot applies $\mathbf N$ to exactly the change to be measured. Part (ii) shows that the swap readout isolates the student. The change in the pre-activation of feature $j$ is the projection $\mathbf s_j^{\top}(\bar{\mathbf h}_{s'}-\bar{\mathbf h}_s)$ of the student's own change onto the determined encoder direction $\mathbf s_j$, the $j$-th row of $\mathbf S$, with no contribution from the teacher or from $\mathbf N$; and the swap readouts of the two students average to the joint pre-activation up to a term that acts only on their difference. A change in firing rate follows directly: with $U=a_{s,j}(x)$ and $\Delta=\mathbf s_j^{\top}(\bar{\mathbf h}_{s'}(x)-\bar{\mathbf h}_s(x))$, expanding $\Pr(U+\Delta>\theta_j)$ to first order in $\Delta$ gives $r_{s'}-r_s\approx p_U(\theta_j)\,\E[\Delta\mid U=\theta_j]$, where $p_U$ is the density of $U$. The change thus measures how far the student's activations move along $\mathbf s_j$ on the inputs where the feature is at the edge of firing, which is the sense in which we call it a reweighting of shared features. Finally, rescaling each checkpoint to the mean norm of the student before training makes the readout blind to a global rescaling of the residual stream, which would otherwise move features across their fixed thresholds.

\paragraph{Verification on our crosscoders.}
We measure these quantities on $2{,}048$ held-out tokens of reasoning text for each of our four three-model crosscoders (Table~\ref{tab:app-swap}). Training leaves $\mathbf N$ at the scale of its random initialization: $\|\mathbf N\|_F/\|\mathbf S\|_F$ is $0.52$--$0.61$, against about $0.5$ at initialization, and the two student encoders of a feature are nearly orthogonal or even anti-correlated, although the activations of the two students differ by only $7$--$12\%$. In the own-slot readout, the undetermined term is as large as the determined one, and the readout shares only about a third of its active features with the swap readout. In a readout that replaces one student slot, the undetermined term is as large as the determined part of the change it measures. In the joint encoding, the undetermined term contributes only $4$--$6\%$ of the pre-activation because it acts only on the small difference between the students, which is why the joint encoding reconstructs all three models well yet cannot attribute a change to either student.

\begin{table}[h]
\centering
\small
\caption{\textbf{The part of the student encoders that training leaves undetermined.} Ratios in the middle block compare the root mean square of the term in $\mathbf N$ with that of the corresponding term in $\mathbf S$.}
\label{tab:app-swap}
\begin{tabular*}{\linewidth}{@{\extracolsep{\fill}}lcccc@{}}
\toprule
 & JustRL & Skywork & R1-7B & Qwen3 \\
\midrule
$\|\mathbf N\|_F/\|\mathbf S\|_F$ (about $0.5$ at initialization) & $0.52$ & $0.61$ & $0.61$ & $0.56$ \\
Cosine of the rows of $\mathbf W_B$ and $\mathbf W_O$ ($0$ at initialization) & $-0.06$ & $-0.24$ & $-0.24$ & $-0.16$ \\
$\|\boldsymbol\delta\|/\|\mathbf m\|$ & $0.12$ & $0.07$ & $0.07$ & $0.11$ \\
\midrule
Own slot: $\mathbf N\,\mathbf h_B$ against $\frac12\mathbf S\,\mathbf h_B$ & $1.04$ & $1.28$ & $1.29$ & $1.16$ \\
Replaced slot: $\mathbf N\,\boldsymbol\delta$ against $\frac12\mathbf S\,\boldsymbol\delta$ & $1.08$ & $1.28$ & $1.27$ & $1.14$ \\
Joint encoding: $\mathbf N\,\boldsymbol\delta$ against $\mathbf S\,\mathbf m$ & $0.06$ & $0.04$ & $0.04$ & $0.06$ \\
\midrule
Jaccard of the active features, own slot and swap readouts & $0.34$ & $0.28$ & $0.28$ & $0.27$ \\
\bottomrule
\end{tabular*}
\end{table}

\section{Experimental Setup}
\label{app:experimental-details}

\subsection{OPD Training}
\label{app:opd}

Section~\ref{sec:analysis} distills DeepSeek-R1-Distill-Qwen-1.5B~\citep{guo2025deepseek} from the three teachers of Table~\ref{tab:settings}, and Section~\ref{sec:further} distills Qwen3-1.7B-Base~\citep{yang2025qwen3} from Qwen3-4B-Base-GRPO~\citep{li2026rethinking}, once from the base student and once from the student after the warm-up. All five OPD runs follow the vanilla OPD recipe of \citet{li2026rethinking} with the same hyperparameters (Table~\ref{tab:app-opd}), implemented in verl~\citep{sheng2024hybridflow}. At each step, the current student samples $4$ rollouts for each of $64$ prompts, and the frozen teacher scores every position of these rollouts. The per-token signal is the reverse KL divergence $D_{\mathrm{KL}}(q\,\|\,p)$ between the student's next-token distribution $q$ and the teacher's $p$, approximated by the log-ratio between the two on the student's $16$ most probable tokens at each position, weighted by the student's probabilities renormalized over these tokens. No other loss or reward is added. Each run makes one pass over DAPO-Math-17k~\citep{yu2026dapo}, which takes $279$ steps, and we analyze its final checkpoint.

\begin{table}[h]
\centering
\small
\caption{\textbf{OPD hyperparameters,} shared by all OPD runs in the paper.}
\label{tab:app-opd}
\begin{tabular}{@{}p{0.4\linewidth}p{0.56\linewidth}@{}}
\toprule
Hyperparameter & Value \\
\midrule
Training prompts & DAPO-Math-17k, one epoch (279 steps) \\
Prompts per step $\times$ rollouts per prompt & $64\times4$ \\
Rollout sampling & temperature 1.0, top-$p$ 1.0 \\
Maximum prompt / response length & 1,024 / 7,168 tokens \\
Teacher signal & log-ratio on the student's top-16 tokens \\
KL penalty, entropy bonus, outcome reward & none \\
Optimizer & AdamW, $\beta_1=0.9$, $\beta_2=0.999$, weight decay 0.01 \\
Learning rate & $10^{-6}$, constant \\
Gradient clipping & 1.0 \\
Optimizer updates per step & 1 \\
Trained parameters & all \\
\bottomrule
\end{tabular}
\end{table}

\subsection{SFT Warm-up}
\label{app:warmup}

Following Simple-OPD~\citep{liu2026simple}, we build the warm-up data from the teacher's own rollouts on the OPD training prompts. We draw a fixed random subset of $2{,}048$ prompts from DAPO-Math-17k, write each as ``\texttt{\{question\} Please reason step by step, and put your final answer within \textbackslash boxed\{\}.}'' in the student's chat template, and sample $4$ responses per prompt from Qwen3-4B-Base-GRPO with temperature $0.6$, top-$p$ $0.95$, top-$k$ $20$, and at most $16{,}384$ tokens. We keep a response if it ends on its own, gives a final answer in \texttt{\textbackslash boxed\{\}}, and does not degenerate into repetition, which leaves $7{,}385$ of the $8{,}192$ responses. Of the $1{,}337$ prompts with at least one correct kept response, we sample $704$ at random, the size of the warm-up set of Simple-OPD, and keep one correct response for each; with its prompt, an example averages $2{,}981$ tokens. Following the recipe of Simple-OPD, we fine-tune the student on these $704$ examples with LoRA~\citep{hu2022lora} of rank $32$ and scaling $\alpha=32$ on all linear layers, a learning rate of $5\times10^{-5}$, and batches of $8$ sequences of at most $16{,}384$ tokens for $175$ steps, about two epochs. After merging the adapter into the weights, we run OPD from this student with the same teacher, prompts, and hyperparameters as the run without the warm-up.

\subsection{Evaluation}
\label{app:eval}

We evaluate on AIME 2024 ($30$ problems), AIME 2025 ($30$), and AMC 2023 ($40$). Every model samples with temperature $0.7$ and top-$p$ $0.95$, up to $31{,}744$ tokens, in the chat template of the base student, and responses that are truncated or give no gradable answer count as incorrect. We report avg@$n$, the accuracy over $n$ samples per problem averaged over the problems of a benchmark. ``Avg.''\ is the mean over the three benchmarks, and ``Recovered'' is the share of the teacher's advantage over the base student in Avg.\ that a student attains. For the settings of Section~\ref{sec:analysis}, we draw $n=256$ samples per problem (Table~\ref{tab:app-eval-r1}); for Section~\ref{sec:further}, we follow Simple-OPD and use $n=8$ (Table~\ref{tab:app-eval-qwen3}, which adds the student after the warm-up alone to Table~\ref{tab:sft-warmup}).

\begin{table}[h]
\centering
\small
\caption{\textbf{Accuracy in the settings of Section~\ref{sec:analysis}}, avg@256 (\%). All OPD students start from the same base student.}
\label{tab:app-eval-r1}
\begin{tabular*}{\linewidth}{@{\extracolsep{\fill}}llccccc@{}}
\toprule
Setting & Model & AIME24 & AIME25 & AMC23 & Avg. & Recovered \\
\midrule
-- & Base student & 29.3 & 23.8 & 71.7 & 41.6 & -- \\
\midrule
\multirow{2}{*}{JustRL}  & OPD student & 46.9 & 36.0 & 86.9 & 56.6 & 80\% \\
                         & Teacher     & 52.6 & 38.4 & 90.4 & 60.5 & -- \\
\midrule
\multirow{2}{*}{Skywork} & OPD student & 39.1 & 30.0 & 79.0 & 49.4 & 26\% \\
                         & Teacher     & 67.2 & 52.1 & 93.9 & 71.1 & -- \\
\midrule
\multirow{2}{*}{R1-7B}   & OPD student & 31.5 & 25.8 & 73.3 & 43.5 & 9\% \\
                         & Teacher     & 54.6 & 41.1 & 90.5 & 62.0 & -- \\
\bottomrule
\end{tabular*}
\end{table}

\begin{table}[h]
\centering
\small
\caption{\textbf{Accuracy in the setting of Section~\ref{sec:further}}, avg@8 (\%), including the student after the warm-up alone.}
\label{tab:app-eval-qwen3}
\begin{tabular*}{\linewidth}{@{\extracolsep{\fill}}lccccc@{}}
\toprule
 & AIME24 & AIME25 & AMC23 & Avg. & Recovered \\
\midrule
Base student  & 4.6  & 2.1  & 22.5 & 9.7  & --   \\
Warm-up       & 9.6  & 6.3  & 40.9 & 18.9 & 35\% \\
OPD           & 7.9  & 4.6  & 40.3 & 17.6 & 30\% \\
Warm-up + OPD & 12.1 & 9.6  & 44.4 & 22.0 & 46\% \\
Teacher       & 22.9 & 20.0 & 65.9 & 36.3 & --   \\
\bottomrule
\end{tabular*}
\end{table}

\subsection{Crosscoders}
\label{app:crosscoders}

All crosscoders share the training data, the activation site, and the hyperparameters in Table~\ref{tab:app-cc}. The training data consist of $400$M tokens, drawn in equal parts from OpenThoughts-114k~\citep{guha2025openthoughtsdatarecipesreasoning}, formatted with the chat template, and from RedPajama-Data-1T-Sample~\citep{weber2024redpajama}, and cut into sequences of $512$ tokens with the tokenizer of each model family. All models of a crosscoder read the same sequences, and we take their residual stream at the output of a middle block, block $14$ of $28$ for the R1-family models and Qwen3-1.7B-Base and block $18$ of $36$ for Qwen3-4B-Base-GRPO (counting from $0$). Each model's activations are centered and divided by a single factor that sets their total variance to $1$, which puts models of different widths on the same scale. The activations are stored in $401$ shards, of which four ($0$, $133$, $267$, and $400$, about $3.0$M tokens) are held out from training; all feature statistics read with the swap readout are computed on them.

Section~\ref{sec:analysis} uses one three-model crosscoder per setting, trained on the base student, the OPD student, and the teacher. Section~\ref{sec:further} uses five crosscoders of Qwen3 models: the three-model crosscoder of the base student, the OPD student, and the teacher, through which we read every student with the swap readout; the three-model crosscoder of the student after the warm-up, its OPD student, and the teacher, which gives the teacher-dominant features after the warm-up; and the two-model crosscoders of the teacher with the student before and after the warm-up, and of the student before and after the warm-up, which give the NRN in Figure~\ref{fig:sft-nrn}.

Table~\ref{tab:app-cc-quality} reports the reconstruction quality of the crosscoders of Section~\ref{sec:analysis}. The swap readout reconstructs each student as well as the joint encoding does. The same holds in Section~\ref{sec:further}: through the three-model crosscoder trained before the warm-up, the swap readout explains $0.755$, $0.760$, $0.756$, and $0.759$ of the variance of the base student, the OPD student, the student after the warm-up, and the student after the warm-up and OPD, measured on $0.2$M held-out tokens.

\begin{table}[h]
\centering
\small
\caption{\textbf{Crosscoder hyperparameters,} shared by all crosscoders in the paper.}
\label{tab:app-cc}
\begin{tabular}{@{}p{0.3\linewidth}p{0.66\linewidth}@{}}
\toprule
Hyperparameter & Value \\
\midrule
Dictionary size & 32,768 features \\
Sparsity & BatchTopK, $k=50$ active features per token on average over a batch \\
Inference & threshold estimated during training (from step 1,000, moving average with rate 0.999) \\
Batch size & 4,096 token positions \\
Optimizer & Adam, $\beta_1=0.9$, $\beta_2=0.999$ \\
Learning rate & $1.41\times10^{-4}$, 1,000 warm-up steps, no decay \\
Training length & about one pass over the training tokens (96,917 scheduled updates) \\
Auxiliary loss & dead-feature loss with weight $1/32$ \\
Precision and seed & float32, seed 42 \\
\bottomrule
\end{tabular}
\end{table}

\begin{table}[h]
\centering
\small
\caption{\textbf{Reconstruction quality of the crosscoders of Section~\ref{sec:analysis}} on the held-out tokens. L0: mean number of active features per token under the joint encoding. FVE: fraction of the variance of a model's activation explained by its reconstruction, for the base student (Base), the OPD student (OPD), and the teacher.}
\label{tab:app-cc-quality}
\begin{tabular*}{\linewidth}{@{\extracolsep{\fill}}lcccccc@{}}
\toprule
& & \multicolumn{3}{c}{FVE, joint encoding} & \multicolumn{2}{c}{FVE, swap readout} \\
\cmidrule(lr){3-5}\cmidrule(l){6-7}
Setting & L0 & Base & OPD & Teacher & Base & OPD \\
\midrule
JustRL  & 50.5 & 0.816 & 0.813 & 0.811 & 0.817 & 0.814 \\
Skywork & 48.9 & 0.814 & 0.822 & 0.767 & 0.814 & 0.821 \\
R1-7B   & 48.7 & 0.814 & 0.824 & 0.755 & 0.815 & 0.823 \\
\bottomrule
\end{tabular*}
\end{table}

\section{Analysis Details}
\label{app:analysis-details}

\subsection{Feature Statistics}
\label{app:feature-stats}

Unless stated otherwise, we read feature statistics with the swap readout on the held-out shards: $3{,}019{,}499$ token positions, every position of the $5{,}909$ held-out sequences except the first, about two thirds of them from OpenThoughts. We process the sequences in chunks of $32$ ($16{,}352$ positions) and, within each chunk, rescale the activations of every student checkpoint so that their mean norm equals that of the student before training. A feature fires at a position when its entry of the code is positive, and its firing count is the number of positions at which it fires. We measure the change in the use of a feature between two checkpoints by $\log_2\frac{c'+1}{c+1}$, where $c$ and $c'$ are its firing counts before and after training.

In Figure~\ref{fig:feature-reweighting}d, we consider the features that fire at least $100$ times in the two checkpoints together, $15{,}218$, $15{,}620$, and $15{,}332$ features under JustRL, Skywork, and R1-7B; the dashed band marks a change by a factor of at most $1.2$. A feature counts as gained or lost when it fires at least $10$ times in one checkpoint and never in the other. The share of the features that change on a token is one minus the Jaccard similarity of the sets of features active before and after training; it averages $0.12$ under JustRL and $0.08$ under Skywork and R1-7B. The JustRL teacher has the student's width, so we read its firing rates by placing its activation, rescaled in the same way, in all three slots.

\subsection{Decision Tokens and Policy Changes}
\label{app:decision}

The decision tokens of Section~\ref{sec:analysis-local} are the words \emph{Wait}, \emph{wait}, and \emph{actually} for reflection; \emph{Hmm}, \emph{maybe}, \emph{perhaps}, and \emph{Alternatively} for hesitation; and \emph{So}, \emph{Therefore}, \emph{Hence}, \emph{Now}, \emph{Let}, and \emph{Similarly} for progression. We count their tokens with and without a leading space and keep those that occur at least $300$ times in the held-out reasoning text. To find which features fire on decision tokens, we take the $12$ held-out positions at which a feature fires most strongly in the student before OPD, and call it a decision-token feature if at least $6$ of them hold a decision token. Figure~\ref{fig:local-reweighting}a ranks the features that fire at least $1{,}000$ times in the two checkpoints together by the absolute change in their use. To compare how many features change at decision tokens and elsewhere, we average the per-token share of changed features over the positions that hold a decision token of each category and divide it by its average over all positions of the held-out reasoning text.

The policy analysis of Section~\ref{sec:analysis-local} runs on the models' own generations rather than on held-out text. From a pool of $1{,}546$ problems (AIME 2024, AIME 2025, AMC 2023, MATH-500, Minerva, and OlympiadBench), we take $200$ problems, preferring those on which the base student produces both correct and incorrect responses, and take $4$ rollouts of the base student for each, sampled as in Appendix~\ref{app:eval} and truncated to $4{,}096$ tokens. We measure the changes in the policy on these rollouts, at every fourth response position, over the vocabulary that the models share. At each such step, the KL divergence from the teacher's next-token distribution to the base student's measures the teacher's signal, and the KL divergence from the OPD student's next-token distribution to the base student's measures how much OPD changes the student. A step emits a category of decision tokens when its next token belongs to that category. The most-changed steps are the $5\%$ with the largest change, and the over-representation of decision tokens among them is the share of these steps that emit a decision token, divided by that share among all steps.

\subsection{Decomposing the Warm-up's Reweighting}
\label{app:decomposition}

In Section~\ref{sec:sft-reweight}, we read all students with the swap readout through the three-model crosscoder trained before the warm-up, on the four held-out shards, and keep the $8{,}524$ features that fire at least $1{,}000$ times in both the base student and the OPD student. For each student, the changes $\log_2\frac{c+1}{c_B+1}$ against the base student's counts $c_B$ form a vector over these features: $\mathbf o$ for the OPD student, $\mathbf s$ for the student after the warm-up, and $\mathbf u$ for the student after the warm-up and OPD. The slope $\beta=\mathbf s^{\top}\mathbf o/\mathbf o^{\top}\mathbf o$ is the least-squares fit of $\mathbf s\approx\beta\,\mathbf o$ through the origin, and the residual $\mathbf r_s=\mathbf s-\beta\,\mathbf o$ carries the share $\|\mathbf r_s\|^2/\|\mathbf s\|^2$ of the warm-up's change. To test whether this residual survives OPD, we fit $\mathbf u$ against $\mathbf o$ in the same way and compute the Spearman correlation across features between its residual $\mathbf r_u$ and $\mathbf r_s$.

A single feature dominates $\mathbf r_s$. It fires $16{,}493$ times in the base student, $15{,}414$ of them on the token ``\,\~{A}'', a garbled character in web text, but only $1{,}092$ times after the warm-up, while OPD leaves it unchanged ($16{,}486$). It alone carries $62\%$ of $\|\mathbf r_s\|^2$, so we report the residual share without it ($56\%$) as well as with it ($77\%$). Excluding it leaves $\beta$ and the Spearman correlation unchanged, and excluding the $10$ or $50$ features with the largest residual instead gives a share of about $55\%$. To describe the residual, we group the features by the share of their firings in the base student that fall on the two held-out shards of reasoning text: without the dominant feature, features with at least $80\%$ of their firings on reasoning text carry $50\%$ of $\|\mathbf r_s\|^2$, features with at least $80\%$ on general text $22\%$, and the remaining features $28\%$. We then read the $15$ features with the largest positive and the $15$ with the largest negative residual whose sign is the same in $\mathbf r_u$, by the tokens on which they fire and by their strongest contexts in the base student and the student after the warm-up.

Both $\mathbf s$ and $\mathbf o$ are measured from the same base student, so features that respond to any perturbation of its activations, or errors in its measured firing rates, could make the two changes look aligned even if the warm-up and OPD were unrelated. Two controls rule this out. The random shift adds the warm-up's change to the base student's activation at every token after rotating it by a fixed random orthogonal matrix $\mathbf Q$, which keeps the size of the change but not its direction:
\begin{equation}
\label{eq:app-rotation}
\bar{\mathbf h}_{\mathrm{rand}}(x)=\bar{\mathbf h}_B(x)+\big(\bar{\mathbf h}_S(x)-\bar{\mathbf h}_B(x)\big)\,\mathbf Q,
\end{equation}
where $\bar{\mathbf h}_B$ and $\bar{\mathbf h}_S$ are the activations of the student before and after the warm-up. Read in the same way, the random shift gives $\beta=-0.03$, and its residual is uncorrelated with $\mathbf r_u$ (Spearman $0.00$). Replacing the warm-up's change at every token by isotropic Gaussian noise of the same norm likewise gives no alignment ($\beta=-0.02$). For the control on disjoint tokens, in which the two changes share no measurement of the base student, we compute $\mathbf s$ on shard $0$ and $\mathbf o$ on shard $133$, the two held-out shards of reasoning text with about $1$M tokens each, which gives $\beta=0.39$, against $0.46$ when both are computed on shard $133$.

\subsection{Feature-Level Intervention}
\label{app:feature-intervention}

The intervention of Section~\ref{sec:sft-reweight} uses the three-model crosscoder trained before the warm-up. At every position except the first, the student before the warm-up, the student after the warm-up, and the teacher read the same prefix, and we compute the swap codes $\mathbf z_B$ and $\mathbf z_S$ of the two students (Eq.~\eqref{eq:swap}). The warm-up's change at that position is $\Delta\mathbf h=\mathbf D_O(\mathbf z_S-\mathbf z_B)/s_O$, where $\mathbf D_O$ is the decoder of the OPD student's slot and $s_O$ the scale of that slot's activation normalization. We rescale $\Delta\mathbf h$ to the mean norm of the student it is applied to and add it to, or subtract it from, that student's residual stream at the output of block~$14$; the student's later blocks then run on the modified stream. Since only a difference of two codes is added, the crosscoder's reconstruction error does not enter the modified activation. During generation, the change is recomputed at every new token from the three models' activations on the prefix generated so far. In the shuffled control, the entries of $\mathbf z_S-\mathbf z_B$ are permuted across features by a fixed random permutation before decoding, which keeps the codes but assigns them to other decoder directions; the resulting change has a similar size ($4.5\%$ of the norm of the residual stream, against $6.2\%$ for the warm-up's change). All conditions in Table~\ref{tab:feature-intervention} sample $8$ responses to each of the $100$ problems with temperature $0.7$ and top-$p$ $0.95$ and up to $8{,}192$ new tokens, and are graded as in Appendix~\ref{app:eval}; almost no correct response of these students is longer. For the $95\%$ bootstrap intervals, we compute for every problem the difference between the accuracy of the modified and the unmodified student over its $8$ samples, draw $100$ problems with replacement $2{,}000$ times, and average the drawn differences within each benchmark and then over the three benchmarks; the interval spans the $2.5$th to the $97.5$th percentile of these $2{,}000$ averages. An interval that excludes zero thus indicates a difference that does not hinge on which problems the benchmarks happen to contain. To read the resulting reweighting, we use that the change is added at the layer the crosscoder reads: on a fixed text, the modified student's activation there is its unmodified activation plus the change. We form this activation on the four held-out shards from the stored activations, read it with the swap readout, and place its change in firing rate in the decomposition of Appendix~\ref{app:decomposition}, which gives $\beta$ and, as the Spearman correlation of its residual with $\mathbf r_s$, $\rho$ in Table~\ref{tab:feature-intervention}.

\section{Feature Visualizations}
\label{app:feature-examples}

This appendix shows what the features behind our findings encode. Except for the teacher's features in Figure~\ref{fig:cards-teacher}, all features are read with the swap readout on the held-out tokens. Each card lists a feature's strongest contexts in the student before training and shades every token by the feature's activation on it; the numbers on the right give the feature's firing counts.

\subsection{The Swap Readout on Held-out Passages}
\label{app:swap-trace}

Figure~\ref{fig:swap-trace} reads the student before and after OPD under JustRL on three held-out passages of reasoning text, for the three decision-token features that OPD changes most (Figure~\ref{fig:local-reweighting}b). The two readouts of a passage differ only in the student's activation, and they show the outcomes of Figure~\ref{fig:swap-outcomes} on real text. The \emph{Wait} feature keeps firing on some occurrences of \emph{Wait} but stops on others; the \emph{So}/\emph{Therefore} feature keeps firing on the words that draw a conclusion but stops on some of the tokens around them; and the \emph{maybe}/\emph{perhaps} feature starts to fire on hedging phrases such as \emph{this suggests that perhaps}. In all three cases, OPD moves the feature's firing count toward the teacher's. The \emph{Wait} feature fires almost only on decision tokens ($86\%$ of its firings on the held-out reasoning text), whereas the other two fire most strongly on the decision words but mostly around them: $18\%$ of their firings fall on decision tokens, nine times the share of these tokens in the text, and their changes after OPD are spread in the same way.

\begin{figure}[h]
\centering
\includegraphics[width=\linewidth]{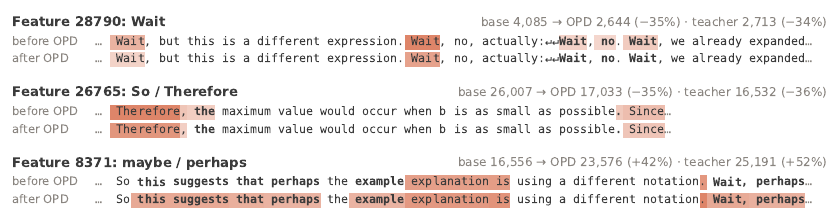}
\caption{\textbf{The swap readout on held-out passages, before and after OPD under JustRL.} Each pair of lines reads one passage in the student before and after OPD for one decision-token feature; shading: the feature's activation on each token; bold: the tokens on which the feature starts or stops firing. Right: the feature's firing count before OPD, after OPD, and in the teacher.}
\label{fig:swap-trace}
\end{figure}

\subsection{Features the Teachers Dominate}
\label{app:cards-teacher}

Figure~\ref{fig:cards-teacher} shows features that the Skywork and R1-7B teachers dominate (Section~\ref{sec:analysis-reweight}). Since the swap readout places a student's activation in the student slots, we read these features with the joint encoding of the three models, on $434$ sequences from the held-out shards and away from their first $16$ positions, where features tied to the position fire. They respond to titles of works, religious texts, LaTeX markup, dates of web posts, place names, and the notation of formulas, and the student's share of their decoder norm is the same before and after OPD.

\begin{figure}[h]
\centering
\includegraphics[width=\linewidth]{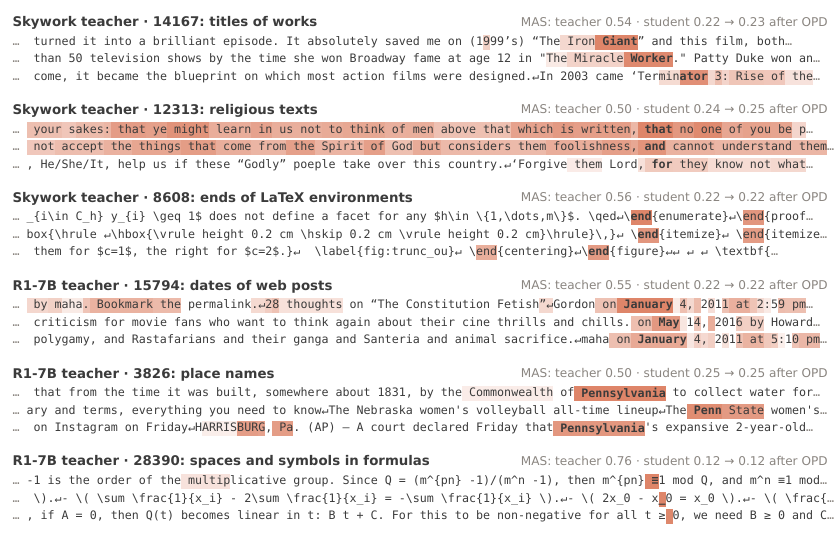}
\caption{\textbf{Features that the Skywork and R1-7B teachers dominate.}
Strongest contexts of each feature under the joint encoding of the three
models, away from the start of the sequence; shading: the feature's activation
on each token; bold: the strongest token. Right: the width-adjusted MAS of the
teacher and of the student before and after OPD.}
\label{fig:cards-teacher}
\end{figure}

\subsection{Decision-Token Features}
\label{app:cards-decision}

Figure~\ref{fig:cards-decision} shows the decision-token features among the $50$ features whose firing rate OPD changes most (Section~\ref{sec:analysis-local}): six under JustRL and two under Skywork. They fire most strongly on the decision words themselves, where the trace turns, and under JustRL, OPD moves each of them toward the teacher's firing count.

\begin{figure}[h]
\centering
\includegraphics[width=\linewidth]{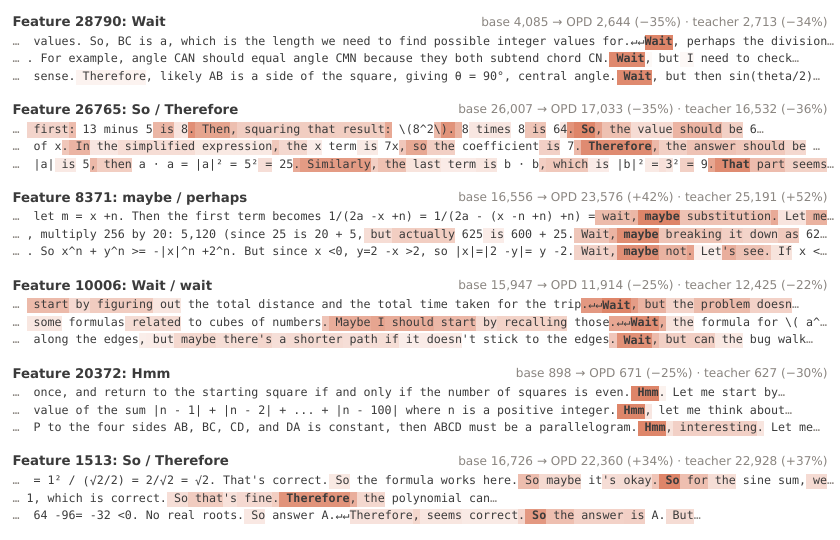}
\caption{\textbf{Decision-token features that OPD reweights under JustRL and Skywork.}
Strongest held-out contexts of each feature in the student before OPD, read
with the swap readout; shading and bold as in Figure~\ref{fig:cards-teacher}.
Right: the feature's firing count before OPD, after OPD, and, under JustRL, in
the teacher.}
\label{fig:cards-decision}
\end{figure}

\subsection{Features the Warm-up Moves along OPD's Direction}
\label{app:cards-along}

Figure~\ref{fig:cards-along} shows six of the $40$ features whose firing rate OPD changes most under Qwen3, three that OPD raises and three that it lowers (Section~\ref{sec:sft-reweight}); $35$ of the $40$ move the same way after the warm-up. The warm-up moves each of the six in OPD's direction by a smaller amount, and after the warm-up and OPD, the change matches or exceeds that of OPD alone. The features OPD and the warm-up raise fire when the trace settles on an answer after a long unresolved search, names a solution method, or ends an attempt that does not work out; those they lower fire on \emph{Wait} at the start of a paragraph, on the step after an equation, and on arithmetic within equations.

\begin{figure}[h]
\centering
\includegraphics[width=\linewidth]{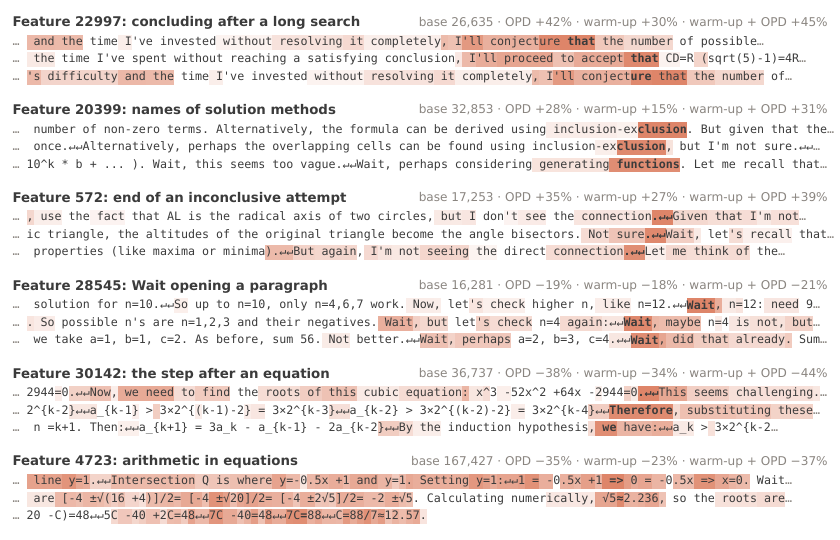}
\caption{\textbf{Features that the warm-up moves along OPD's direction under
Qwen3.} Contexts and shading as in Figure~\ref{fig:cards-decision}. Right: the
feature's firing count in the student before the warm-up and its change after
OPD, after the warm-up, and after the warm-up and OPD.}
\label{fig:cards-along}
\end{figure}

\subsection{Features the Warm-up Moves beyond OPD's Direction}
\label{app:cards-warmup}

Figure~\ref{fig:cards-warmup} shows six of the features that carry the warm-up's reweighting outside OPD's direction under Qwen3 (Section~\ref{sec:sft-reweight}): two for the conversation format, three for the style of reasoning, and one for mathematical notation. Unlike those of Figure~\ref{fig:cards-along}, their change after the warm-up is not a smaller version of OPD's: it falls where OPD barely moves them, runs against OPD, or goes further than OPD. Features for \emph{Wait} appear in both figures because a feature's change can lie partly along OPD's direction and partly outside it: the warm-up lowers these features in OPD's direction, but further than the slope $\beta$ predicts.

\begin{figure}[h]
\centering
\includegraphics[width=\linewidth]{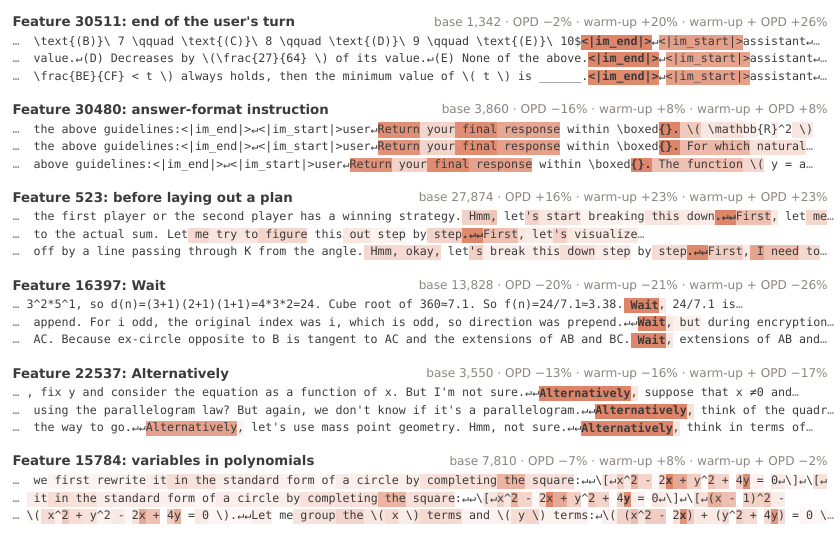}
\caption{\textbf{Features that carry the warm-up's reweighting outside OPD's
direction under Qwen3.} Contexts and shading as in
Figure~\ref{fig:cards-decision}. Right: the feature's firing count in the
student before the warm-up and its change after OPD, after the warm-up, and
after the warm-up and OPD.}
\label{fig:cards-warmup}
\end{figure}

\end{document}